\documentclass[letterpaper,10pt,conference]{ieeeconf}
\IEEEoverridecommandlockouts
\usepackage[T1]{fontenc}
\usepackage{amsmath,amssymb,graphicx,booktabs,multirow,cite,url}
\newtheorem{proposition}{Proposition}
\usepackage{xcolor}
\newcommand{\dev}{\bar\delta}
\newcommand{\acc}{\delta_{\mathrm{acc}}}
\newcommand{\rej}{\delta_{\mathrm{rej}}}

\title{\LARGE\bf Predict Before You Deploy: Offline Prediction of Quantization-Induced Task Degradation for World Action Models}
\author{
Jiuyi Xu$^{1}$, Jinjia Guo$^{2}$, Meida Chen$^{3}$, Jing Du$^{2}$, Yangming Shi$^{1,4}$
\thanks{$^{1}$Robotics Program, Colorado School of Mines, Golden, CO USA.}
\thanks{$^{2}$Department of Civil and Coastal Engineering, University of Florida, Gainesville, FL USA.}
\thanks{$^{3}$University of Southern California - Institute for Creative Technology, Los Angeles, CA USA.}
\thanks{$^{4}$Department of Civil and Environmental Engineering, Colorado School of Mines, Golden, CO USA.}
}
\begin{document}
\maketitle
\thispagestyle{empty}\pagestyle{empty}

\begin{abstract} 

World action models (WAMs) rely on video-generation backbones, requiring substantial memory and compute for deployment.
Post-training quantization reduces memory and can accelerate inference, but bit width, grouping, and quantizer choice define a large configuration space.
Identifying configurations that preserve task performance through exhaustive closed-loop evaluation is costly.
We propose PreDE (Predict Before You Deploy), a policy-calibrated framework for predicting quantization-induced task degradation from offline action deviations.
Using closed-loop outcomes from a small development set, PreDE calibrates two thresholds and accepts, rejects, or defers new configurations using a fixed observation log.
Under a within-setting label-ordering hypothesis, the rule issues decisions where all thresholds consistent with the development labels agree.
Across five WAMs and four benchmark settings, quantization produces configuration-dependent task losses that cannot be explained by bit width alone or a shared deviation threshold.
Across 28 held-out configurations from two policies, PreDE issued 21 decisions before observing closed-loop outcomes (75\% coverage), all matching the observed acceptable or degraded labels.
Deferred candidates included both acceptable outcomes and a 33-percentage-point loss.
In 450 Franka Research 3 trials across two independently fine-tuned policies, all configurations assigned to high-deviation groups before testing showed significant degradation, while low-deviation comparisons showed no statistically significant degradation.
On the real robot, W4A4 achieved a $1.37\times$ action-query speedup and approximately 44\% lower peak memory.
These results support policy-specific behavioral calibration for quantization configuration selection while identifying candidates that require closed-loop evaluation.
 The code is available at \url{https://github.com/jiuyixu25/PreDE}.

\end{abstract}

\section{Introduction}
\label{sec:introduction}

World action models (WAMs) combine video generation and action modeling for robot control~\cite{cosmos,uva,mimic,rynn,fastwam}, but their large backbones impose substantial memory and computational demands on deployment.
Post-training quantization (PTQ)~\cite{svdquant,hqq,qserve,viditq} can reduce these demands without retraining the policy.
However, its deployment value depends on retaining adequate task performance.
Selecting a quantization configuration involves weight and activation precision, grouping, layer coverage, and quantizer design.
These choices produce candidates with different resource requirements and behavioral effects, even at the same nominal bit width.
Evaluating every candidate in closed loop reveals its task performance but requires repeated simulator or robot interaction.
We therefore ask: \emph{can we predict whether a new quantization configuration will degrade task performance before evaluating it in closed loop?}

Offline action comparison provides an accessible signal: replay identical observations through the reference and quantized policies and measure their action deviations.
However, task success depends on the consequences of action changes throughout a rollout, and deviation magnitudes need not have the same significance across policies.
Our experiments show sharply different outcomes for grouped and per-channel three-bit Cosmos configurations, as well as a cross-policy ordering reversal: at three bits, Fast-WAM suffers a larger task loss than Cosmos despite a smaller mean action deviation.
Within the initial Cosmos and UVA development sets, acceptable configurations nevertheless have smaller deviations than degraded configurations.
These observations motivate interpreting offline deviation through behavioral evidence specific to the policy and evaluation setting.

We propose PreDE (\textbf{Pre}dict Before You \textbf{DE}ploy), a policy-calibrated framework for predicting quantization-induced task degradation.
A small development set pairs offline action deviations with closed-loop outcomes to establish acceptance and rejection thresholds.
For each new candidate, PreDE compares its actions with cached reference actions on a fixed observation log.
It accepts candidates at or below the acceptance threshold, rejects those at or above the rejection threshold, and defers intermediate candidates to closed-loop evaluation.
After calibration, issuing predictions requires only offline policy queries.  
Deferred candidates require further interaction if a deployment decision is needed.
We provide a post hoc formalization of this rule through a within-setting single-crossing hypothesis:
acceptable and degraded outcome labels can be separated by an unknown deviation threshold, while unresolved outcomes impose no constraint.
Under this hypothesis, the two thresholds delimit the interval of thresholds consistent with the development labels.
PreDE decides where these thresholds agree and defers where they disagree.
This interpretation explains the max/min construction and deferral interval, with conclusions conditional on the ordering hypothesis.

We evaluate PreDE through cross-model quantization measurements, held-out prediction, and physical robot experiments.
The broad evaluation spans five WAMs and four benchmark settings.
Across 20 Cosmos Policy and eight UVA held-out configurations, PreDE issues 21 decisions (75\% coverage), all matching the observed acceptable or degraded outcome labels.
These predictions are determined before observing candidate closed-loop outcomes.
Deferred candidates include both acceptable outcomes and a 33-percentage-point loss, demonstrating that deferral does not imply mild degradation.
In a separate 450-trial Franka Research 3 experiment, all tested configurations assigned to high-deviation groups before trials show significant degradation.
This physical grouping study provides complementary behavioral evidence without transferring the simulation thresholds.
On the robot, W4A4 quantization of Cosmos Policy provides a $1.37\times$ action-query speedup and approximately 44\% lower peak memory.

Our contributions are threefold:
\begin{itemize}
\item We propose PreDE for predicting quantization-induced task degradation from offline action deviations, and formalize its acceptance, rejection, and deferral regions under a within-setting label-ordering hypothesis.

\item We evaluate predictions on 28 held-out configurations from two WAM policies, with decisions determined before observing closed-loop results, and report prediction results, decision coverage, label ordering, and measurement sensitivity.

\item We characterize quantization outcomes across five WAMs and four benchmark settings, and present a separate real-robot study of behavioral degradation and the latency and memory benefits of quantization.
\end{itemize}

\section{Related Work}
  \label{sec:related}

\paragraph{Robot foundation policies and world action models}
Generalist robot policies such as OpenVLA~\cite{openvla}
and Octo~\cite{octo} learn control from diverse demonstrations.
World-model-based policies further connect actions with
visual dynamics.
Cosmos Policy adapts a pretrained video diffusion model
to generate actions and predict future states~\cite{cosmos}.
UVA learns a joint video--action representation with
separate decoding paths~\cite{uva}, while mimic-video
uses video representations for action prediction~\cite{mimic}.
RynnVLA-002 integrates vision-language-action learning
with world modeling~\cite{rynn}, and Fast-WAM investigates
the role of future prediction during training and
inference~\cite{fastwam}.
We evaluate quantization across these WAMs, whose
architectures and action interfaces motivate
policy-specific behavioral assessment.

\paragraph{Post-training quantization}
PTQ methods reduce numerical precision while controlling
the resulting approximation error.
Round-to-nearest (RTN) provides a basic rounding rule
for uniform quantization~\cite{nagel2021}.
GPTQ uses approximate second-order information for
weight quantization~\cite{gptq}, and AWQ uses activation
statistics to identify and protect salient weight
channels~\cite{awq}.
SmoothQuant redistributes quantization difficulty between
weights and activations through an equivalent scaling
transformation~\cite{smoothquant}.
HQQ provides weight quantization without calibration
data~\cite{hqq}.
For diffusion models, Q-Diffusion addresses quantization
of iterative denoising networks~\cite{qdiffusion},
while PTQD models and corrects quantization noise
within the denoising process~\cite{ptqd}.
PTQ4DiT addresses salient channels and timestep-dependent
activation variation in diffusion transformers~\cite{ptq4dit}.
SVDQuant absorbs outliers into a low-rank component~\cite{svdquant},
and ViDiT-Q studies diffusion-transformer quantization
for image and video generation~\cite{viditq}.
These methods construct low-precision candidates and  
we study their effects on closed-loop task outcomes.

\paragraph{Quantization of robot policies}
Recent methods address quantization for robot control.
QuantVLA introduces scale-calibrated PTQ for VLA
backbones and diffusion action heads~\cite{quantvla}.
Mix-QVLA uses task-evidence distortion and
execution-dependent sensitivity to guide mixed-precision
allocation~\cite{mixq}.
SQIL incorporates saliency into quantization-aware
imitation learning~\cite{sqil}.
These approaches optimize compressed policies.
Deferral follows the principle of withholding predictions
when evidence is insufficient~\cite{chow1970,selective}. 
In PreDE, this differs from rejecting a configuration
as likely to degrade.
Prior work on learning from human demonstrations also
shows that offline validation loss can be unreliable
for selecting policies by task success~\cite{robomimic}.
PreDE instead judges new configurations produced by
existing quantizers.
It uses development closed-loop outcomes to calibrate
offline action deviation, then accepts, rejects, or
defers unseen candidates.
We evaluate predictions fixed before candidate outcomes,
reporting decision coverage and unresolved results.

\section{Deviation Tolerance in Closed Loop}
\label{sec:tolerance}

\paragraph{Observation}
Our broad quantization evaluation shows that nominal precision alone does not determine task performance:
grouped and per-channel three-bit Cosmos configurations produce substantially different outcomes (Sec.~\ref{sec:cross_model}).
Offline action deviation provides a complementary signal.
In the initial Cosmos development set, the largest deviation among acceptable configurations is $0.02191$, while the smallest among degraded configurations is $0.09736$.
The corresponding UVA development values are $0.0049$ and $0.00925$.
These development results exhibit a separation between acceptable and degraded configurations.
The interpretation of deviation nevertheless depends on the evaluation setting: Fast-WAM W3 degrades at a smaller deviation than Cosmos HQQ W3 tolerates.
These observations motivate the within-setting ordering hypothesis below. 
The hypothesis is a post hoc formalization of the method. 

\paragraph{Hypothesis H (single crossing within a setting)}
For a fixed policy, task distribution, control condition, offline measurement protocol, and outcome-labeling rule, there exists a threshold $\theta$ such that every configuration labeled acceptable has $\dev(q;D)\leq\theta$ and every configuration labeled degraded has $\dev(q;D)>\theta$. 
Unresolved outcomes place no constraint on $\theta$. 
An acceptable configuration with deviation greater than or equal to that of a degraded configuration contradicts H within the same setting. 

\paragraph{A sufficient condition for tolerance}
A conditional robustness argument motivates tolerance to small perturbations. 
Suppose the reference and quantized rollouts start from the same state. 
Let $e\geq0$ bound the action perturbation over the rollout, and let $K_H>0$ bound its effect on state discrepancy over an executed chunk of $H$ steps. 
Assume the discrepancy at query $t$ satisfies $r_{t+1}\leq\rho r_t+K_H e$, with $r_0=0$ and $0\leq\rho<1$. 
Then $r_t\leq K_H e/(1-\rho)$. 
If remaining within distance $m$ of the reference trajectory preserves task success, the condition $K_H e/(1-\rho)<m$ is sufficient to preserve a successful reference rollout. 
An additional uniform bound $e\leq c\,\dev(q;D)$ for some $c>0$ would yield the sufficient condition $\dev(q;D)<m(1-\rho)/(K_Hc)$. 
PreDE therefore calibrates its decision thresholds empirically.

\paragraph{Consequences for a decision rule}
Let $C$ be a development set with closed-loop labels (Sec.~\ref{sec:method}), $C_A$ and $C_R$ its acceptable and degraded members, and
\begin{equation}
\acc=\max_{q\in C_A}\dev(q;D),\qquad \rej=\min_{q\in C_R}\dev(q;D).
\label{eq:anchors_def}
\end{equation}
Under this threshold model, when $C_A$ and $C_R$ are nonempty and $\acc<\rej$, the max/min thresholds are exactly the endpoints of the interval of thresholds consistent with the definite development labels, $\Theta(C)=[\acc,\rej)$.

\begin{proposition}[Agreement region]
\label{prop:agreement}
Assume H holds for the setting and $\Theta(C)$ is nonempty. 
Then $\theta\in\Theta(C)$. 
For $\dev(q;D)\leq\acc$, every consistent threshold excludes a degraded label. 
For $\dev(q;D)\geq\rej$, every consistent threshold excludes an acceptable label. 
For $\acc<\dev(q;D)<\rej$, consistent thresholds place the candidate on different sides of the threshold. 
Hence a rule that accepts at or below $\acc$, rejects at or above $\rej$, and defers in between accepts no configuration labeled degraded and rejects none labeled acceptable.
\end{proposition}
\begin{proof}
Acceptable labels imply $\acc\leq\theta$, and degraded labels imply $\theta<\rej$. 
Thus $\dev\leq\acc$ excludes a degraded label, while $\dev\geq\rej$ excludes an acceptable label. 
For any $\acc<x<\rej$, the consistent thresholds $\theta_1=\acc$ and $\theta_2=x$ place $x$ on opposite sides of the threshold rule.
\end{proof}


\begin{proposition}[Refinement]
\label{prop:refine}
Suppose both development subsets are nonempty and $\Theta(C)=[\acc,\rej)$ is nonempty. 
Adding an evaluated configuration never enlarges $\Theta(C)$.
The inclusion is strict if and only if the new outcome is acceptable with $\dev>\acc$ or degraded with $\dev<\rej$. An unresolved outcome leaves the set unchanged. 
A contradictory label can make the consistent set empty.
\end{proposition}

With fixed labels and deviations, additional development evidence can tighten the interval or expose an inconsistency. 
Under H, tightening can increase coverage without introducing a disagreement with a definite outcome label. 

\section{Method}
\label{sec:method}

\begin{figure*}[t]
\centering
\includegraphics[width=\textwidth]{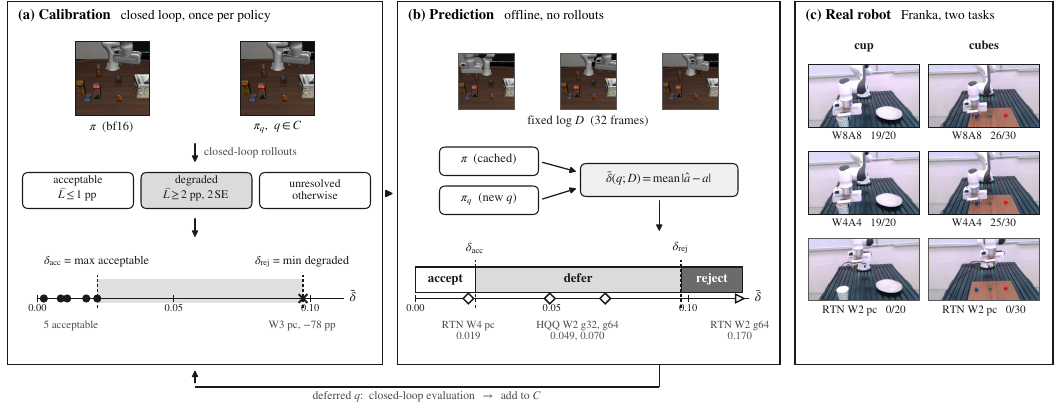}
\caption{\textbf{Overview of PreDE.}
(a) Closed-loop outcomes and offline action deviations of development configurations establish policy-specific
acceptance and rejection thresholds.
(b) New candidates are compared with cached reference
actions on a fixed observation log and accepted, rejected,
or deferred. 
(c) A separate real-robot experiment evaluates deviation groups fixed before trials.}
\label{fig:pipeline}
\end{figure*}

\subsection{Problem formulation and offline measurement}

Given a reference policy $\pi$, a quantization configuration $q$ specifies precision, grouping, layer coverage, and quantizer, producing a quantized policy $\pi_q$.
Under a fixed task distribution and control condition, let $p_0$ and $p_q$ denote their closed-loop success probabilities, with task loss $L(q)=p_0-p_q$.
The condition fixes observation preprocessing, action parameterization, chunk length, execution rate, and denoising steps.
Our goal is to predict degradation for a new configuration before evaluating it in closed loop.
PreDE uses a fixed observation log $D$ and a development set $C$ of quantized configurations with measured closed-loop outcomes (Fig.~\ref{fig:pipeline}).

For each observation in $D=\{o_i\}_{i=1}^{M}$, the reference and candidate receive identical inputs and matched action-generation seeds.
Let $a^{\mathrm{ref}}_{ihj}$ and $a^q_{ihj}$ denote their actions at chunk step $h$ and component $j$.
We measure mean absolute component deviation:
\begin{equation}
\dev(q;D)=
\frac{1}{MHd}
\sum_{i=1}^{M}\sum_{h=1}^{H}\sum_{j=1}^{d}
\left|a^q_{ihj}-a^{\mathrm{ref}}_{ihj}\right|,
\label{eq:metric}
\end{equation}
where $H$ is the chunk length and $d$ is the number of included action components.
The action representation and included components remain fixed within each calibration setting.
Reference actions are cached, so a new candidate requires $M$ complete policy queries without environment interaction.

\subsection{Policy-specific calibration}

For each $q\in C$, we measure $\dev(q;D)$ and the observed task loss $\widehat L(q)=\hat p_0-\hat p_q$, using reference and candidate evaluations under matched conditions.
Before calibration, we specify an acceptable observed-loss tolerance $\epsilon_A$, a degradation threshold $\epsilon_R>\epsilon_A$, and an evidence criterion $E(q)$.
These define
\begin{align}
C_A &=
\{q\in C:\widehat L(q)\leq\epsilon_A\},
\label{eq:CA}\\
C_R &=
\{q\in C:\widehat L(q)\geq\epsilon_R
\land E(q)=1\}.
\label{eq:CR}
\end{align}
Configurations in $C_A$ are \emph{acceptable}, those in $C_R$ are \emph{degraded}, and all others are
\emph{unresolved}.
The same labeling rule applies to held-out outcomes.
These labels describe finite-sample observations: an acceptable label does not establish a population loss bound, and failure to detect degradation does not establish an acceptable outcome.

Our simulation prediction experiments use
$\epsilon_A=0.01$ and $\epsilon_R=0.02$, corresponding
to one and two percentage points.
These specify the acceptable observed-loss tolerance
and the minimum loss required for a degraded label,
respectively, and were fixed before held-out evaluation.
A loss exceeding $\epsilon_R$ is not sufficient for
a degraded label: it must also satisfy the evidence
criterion
\begin{equation}
\widehat L(q)>
2\sqrt{
\frac{\hat p_0(1-\hat p_0)}{n_0}
+
\frac{\hat p_q(1-\hat p_q)}{n_q}
},
\end{equation}
where $n_0$ and $n_q$ are the reference and candidate
episode counts.
Thus, the evidence requirement depends on the observed
success rates and sample sizes, rather than a fixed
minimum detectable loss.

We compute the max/min thresholds in Eq.~\eqref{eq:anchors_def} from these development labels.
The acceptance threshold is the largest deviation among acceptable configurations, while the rejection threshold is the smallest among degraded configurations.
Unresolved outcomes establish neither threshold. 
When both subsets are nonempty and $\acc<\rej$, they define the consistent-threshold interval $\Theta(C)=[\acc,\rej)$ of Sec.~\ref{sec:tolerance}.

PreDE calibrates the behavioral interpretation of deviation. 
The underlying PTQ method determines the quantizer's parameters.
Using actual quantized configurations connects the measurement to observed task outcomes without assuming that synthetic action noise reproduces quantization effects.
Hypothesis H formalizes the ordering needed for the conditional interpretation in Proposition~\ref{prop:agreement}.
Its applicability to new configurations is evaluated through held-out predictions.

\subsection{Prediction and evaluation}

We freeze the log, measurement protocol, and thresholds before closed-loop evaluation of new configurations.
For $q\notin C$, PreDE computes its offline deviation and issues
\begin{equation}
g(q)=
\begin{cases}
\mathrm{accept}, & \dev(q;D)\leq\acc,\\
\mathrm{reject}, & \dev(q;D)\geq\rej,\\
\mathrm{defer},  & \acc<\dev(q;D)<\rej.
\end{cases}
\label{eq:decision}
\end{equation}
If either development subset is empty or $\acc\geq\rej$, PreDE defers all candidates.

Acceptance predicts an acceptable outcome, while rejection predicts degradation.
Under H, Proposition~\ref{prop:agreement} excludes the opposite definite label for each issued decision, but it does not exclude an unresolved outcome.
Deferral makes no outcome prediction and calls for closed-loop evaluation if a deployment decision is needed.
This follows the selective prediction principle of withholding decisions when evidence is insufficient rather than deciding on every candidate~\cite{selective}.

Predictions are determined before observing candidate closed-loop outcomes.
Decision coverage is the fraction of held-out candidates receiving accept or reject, including those whose eventual outcomes remain unresolved.
We score predictions against the labels defined by Eqs.~\eqref{eq:CA}--\eqref{eq:CR}.
Accepting a degraded configuration is a false acceptance, while 
rejecting an acceptable configuration is a false rejection.
Accepting an acceptable configuration or rejecting a
degraded configuration counts as a correct decision.
An issued decision with an unresolved outcome counts
as neither correct nor incorrect.
Deferred outcomes are reported but not scored.

Evaluated configurations may subsequently enter $C$ to update the anchors, as described in
Proposition~\ref{prop:refine}.
Once used for recalibration, they become development
evidence, and the updated rule requires new held-out
evaluation.
This separates initial calibration, offline prediction,
and further interaction for deferred candidates.

\begin{table*}[t]
  \centering\small
  \caption{Closed-loop success rates (\%) for 33 quantized
  model--benchmark pairs.
  W8A8 uses ViDiT-Q-based quantization.  W4A4 uses
  SVDQuant-style quantization. W4 and W3 use HQQ
  weight-only quantization.
  In an additional single-seed Cosmos/LIBERO development
  comparison, RTN W3 achieves 97.40\% success with group
  size 128 and 20.05\% with per-channel quantization.}
  \label{tab:cross}
  \setlength{\tabcolsep}{7pt}
  \begin{tabular}{llrrrrr}
  \toprule
  Model & Benchmark & bf16 & W8A8 & W4A4 & W4 & W3\\
  \midrule
  Cosmos-2B & LIBERO (four suites)
  & 98.38 & 98.12 & 97.98 & 98.20 & 97.87\\
  Cosmos-2B & RoboCasa-24
  & 67.03 & 66.42 & 65.58 & 65.58 & 61.61\\
  UVA-0.5B & LIBERO-10
  & 89.60 & 91.20 & 89.20 & 88.40 & 78.60\\
  UVA-0.5B & PushT
  & 97.47 & 99.61 & 95.79 & 96.02 & 96.79\\
  UVA-0.5B & PushT-M
  & 82.71 & 75.07 & 78.66 & 81.77 & 77.71\\
  mimic-2B & LIBERO-spatial
  & 90.40 & 91.40 & 88.30 & 92.60 & 90.20\\
  mimic-2B & LIBERO-object
  & 94.20 & 92.40 & 89.60 & 93.80 & 95.20\\
  RynnVLA-7B & LIBERO (four suites)
  & 95.30 & --- & --- & 95.05 & 94.75\\
  Fast-WAM-6B & LIBERO (four suites)
  & 95.55 & --- & 95.75 & 95.55 & 90.15\\
  \bottomrule
  \end{tabular}
  \end{table*}

\section{Experiments}
\label{sec:experiments}

\subsection{Experimental Setup}
\label{sec:setup}

\paragraph{Models and configurations}
We evaluate Cosmos Policy~\cite{cosmos}, UVA~\cite{uva},
mimic-video~\cite{mimic}, RynnVLA-002~\cite{rynn},
and Fast-WAM~\cite{fastwam}.
Simulation settings include LIBERO~\cite{libero},
RoboCasa-24~\cite{robocasa}, PushT~\cite{diffusionpolicy},
and a multi-task PushT variant.
Configurations use RTN~\cite{nagel2021} and HQQ weight-only
quantization~\cite{hqq}, ViDiT-Q-based INT8~\cite{viditq},
SVDQuant-style W4A4~\cite{svdquant}, and
QServe-style W4A8~\cite{qserve}.
The main Cosmos comparison quantizes 280 DiT linear layers
while retaining conditioning-related and final projections
in higher precision.

\paragraph{Evaluation settings}
We report success rates and differences from the corresponding bf16 reference. 
All evaluations follow the process specified on each policy’s official website.
The results in Sec.~\ref{sec:cross_model} are averaged over three seeds. 
Reference and candidate measurements use matched action-generation random states.
Resource measurements compare configurations on the same hardware with other inference stages fixed.

\paragraph{Prediction assessment}
Outcomes are labeled acceptable, degraded, or unresolved
using the criteria in Sec.~\ref{sec:method}.
We report all candidates, including deferred configurations, and assess issued decisions jointly with coverage.
Unresolved outcomes count as neither correct nor incorrect.
Two-sided Fisher tests~\cite{fisher1922interpretation}
provide an additional descriptive comparison without
replacing the recorded labeling rule.
We use $g$ for quantization group size, pg for per-group quantization and pc for
per-channel quantization.

\subsection{Quantization Results Across Models}
\label{sec:cross_model}

Table~\ref{tab:cross} summarizes 33 quantized
model--benchmark pairs, of which 19 lose at most one
observed percentage point relative to their reference.
On Cosmos/LIBERO, W8A8, W4A4, HQQ W4, and HQQ W3
achieve 98.12\%, 97.98\%, 98.20\%, and 97.87\% success,
respectively, against 98.38\% for bf16.
The additional single-seed development comparison shows
the importance of granularity: RTN W3 achieves 97.40\%
with group size 128 but only 20.05\% with per-channel
quantization.
Thus, configurations with identical nominal precision
can have substantially different closed-loop outcomes.

Offline deviation also requires setting-specific interpretation.
Fast-WAM W3 loses 5.4 percentage points at $\dev=0.0154$,
whereas Cosmos HQQ W3 loses only 0.51 points on LIBERO
at the larger deviation of 0.0179.
A shared deviation threshold cannot accept the latter
while rejecting the former.
The corresponding HQQ W3 configuration loses 5.42 points
on Cosmos/RoboCasa-24, further illustrating dependence
on the checkpoint and evaluation setting.
These comparisons also differ in task distributions,
action representations, and measurement coverage. 
Some historical deviations cover one suite while success
rates aggregate four.
They motivate setting-specific calibration rather than
isolating policy identity as the sole source of variation.

\subsection{Held-Out Degradation Prediction}
\label{sec:heldout}

We evaluate predictions on 20 Cosmos and eight UVA
configurations on LIBERO-10.
Figure~\ref{fig:heldout} orders candidates by offline
deviation and displays their predicted decisions and
closed-loop outcomes.
Predictions are determined before observing the
corresponding candidate outcomes.

\begin{figure*}[t]
\centering
\includegraphics[width=\textwidth]{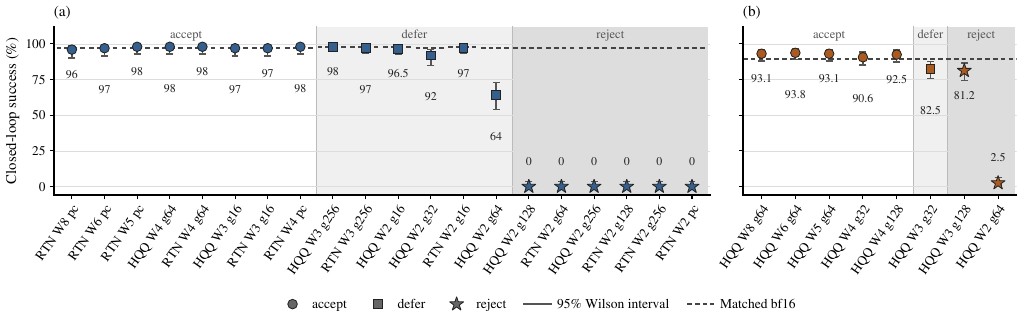}
\caption{Held-out predictions and closed-loop outcomes on
LIBERO-10 for (a) 20 Cosmos and (b) eight UVA candidates,
ordered by offline deviation within each subfigure.
Cosmos decisions are consistent with thresholds
$(0.02191,0.09736)$, while UVA thresholds are $(0.0049,0.00925)$.
Error bars are marginal 95\% Wilson intervals for success,
not intervals on differences from bf16. 
Post hoc binary baselines accept configurations with
$\dev\leq0.009$ or weight precision $\geq4$ bits,
respectively, and reject the rest. 
At 100\% coverage, they yield six and five false rejections, respectively. Both have no false acceptances
and three unresolved outcomes.
}
\label{fig:heldout}
\end{figure*}

\paragraph{Cosmos Policy}
Calibration uses results aggregated over four LIBERO suites, including LIBERO-10. Since LIBERO-10 is part of the calibration task set, we reuse the calibrated thresholds without modification for held-out quantization configurations evaluated on this suite.
Six initial development configurations establish
$\acc=0.02191$ and $\rej=0.09736$.
The 20 held-out candidates vary weight precision,
grouping, and quantizer.
PreDE accepts eight, rejects six, and defers six,
giving 70\% decision coverage (Fig.~\ref{fig:heldout}a).
All eight accepted candidates achieve 96--98 successes
in 100 episodes against the matched bf16 reference of 97/100.
All six rejected candidates achieve 0/100.
Thus, all 14 issued decisions match their observed labels.
RTN W4 pc is the closest accepted candidate to $\acc$:
its deviation is 0.01940, approximately 11\% below the
threshold, and it achieves 98/100.


\paragraph{UVA}
Four development configurations establish
$\acc=0.0049$ and $\rej=0.00925$.
Each of eight held-out candidates receives 160 episodes
across two evaluation batches, with a matched bf16
reference of 143/160.
PreDE accepts five candidates, rejects two, and defers
one, giving 87.5\% coverage (Fig.~\ref{fig:heldout}b).

The accepted HQQ W4 $g{=}128$ lies 2.9\% below $\acc$
and achieves 148/160.
The rejected HQQ W3 $g{=}128$ lies 6.3\% above $\rej$
and achieves 130/160, an 8.1-point loss.
It meets the recorded degradation criterion ($z=2.07$),
although Fisher's test gives $p=0.057$.
The other rejected candidate, HQQ W2 $g{=}64$,
achieves 4/160.
The accepted HQQ W8 $g{=}64$ achieves 149/160,
exceeding the reference by 3.75 percentage points and
receiving an acceptable label.
The deferred HQQ W3 $g{=}32$ remains unresolved at 132/160.
All seven issued decisions match their observed labels.

Across both policies, PreDE issues 21 decisions for
28 candidates (75\% coverage).
All 21 match their acceptable or degraded outcome labels,
with no observed false acceptance or false rejection.
The seven deferred candidates comprise three acceptable,
one degraded, and three unresolved outcomes.

\paragraph{Label ordering}
A post hoc check finds no acceptable--degraded ordering
contradictions within or across the development and
held-out sets.
For Cosmos, the pooled maximum acceptable deviation
is 0.04935 (RTN W2 $g{=}16$), and the minimum degraded
deviation is 0.06955 (HQQ W2 $g{=}64$).
The nearby HQQ W2 $g{=}32$ is unresolved and sets
neither bound.
UVA's corresponding bounds are $(0.0049,0.00925)$.
These describe pooled outcomes; the Cosmos comparison
combines four-suite development results with LIBERO-10
held-out results.

\subsection{Sensitivity Analysis}
\label{sec:sensitivity}

We repeat measurements on the fixed Cosmos log with
five matched action-generation seeds (195--199).
RTN W4 pc remains accepted under the frozen initial
anchors in all five runs, while HQQ W2 $g{=}64$
remains deferred.
Their deviation coefficients of variation are 3.2\%
and 1.2\%, respectively.
The two initial anchor configurations have coefficients
of variation of 2.2\% and 0.7\%. 
Their remeasurements
fall on both sides of the corresponding frozen boundaries.
Remeasuring anchors and candidates together still accepts
RTN W4 pc under every seed.

A paired bootstrap jointly resamples the log's eight
trajectories for the candidate and acceptance anchor.
Across seeds, the 95\% interval lower bounds for the
anchor deviation minus the RTN W4 pc deviation range
from 0.0008 to 0.0012 and are all positive.
These measurements support this candidate's stability
on the tested log, without establishing robustness to
independently collected logs, different log lengths, or UVA.

A post hoc analysis of the fixed initial Cosmos
development set finds unchanged anchors when
$\epsilon_A$ ranges from the largest acceptable
development loss to values strictly below the fixed
$\epsilon_R=0.02$, with the evidence criterion unchanged.
This reflects the separation of the development losses,
rather than general insensitivity to labeling tolerances.
Unchanged anchors also do not imply unchanged held-out
labels or scoring.
All reported prediction results use the original
$\epsilon_A=0.01$ and $\epsilon_R=0.02$.

\subsection{Real-World Evaluation}
\label{sec:real_world}

\paragraph{Tasks and protocol}
Two independently fine-tuned Cosmos policies perform
cup pick-and-place and instruction-conditioned colored-cube
selection on a Franka Research 3 with third-person and
wrist cameras.
An RTX~5080 server generates 16-step action chunks with
a configured execution rate of 15\,Hz, distinct from
the policy-query frequency.
Each task evaluates bf16 and eight quantized configurations,
with 20 cup trials and 30 cube trials per configuration
(450 trials total).

Offline deviation averages differences in the three
translational action components over 30 observation frames,
expressed in millimeters (mm).
Low- and high-deviation groups are fixed before trials
from gaps of 1.75--6.32\,mm for cups and
1.33--7.75\,mm for cubes.
No closed-loop anchors are calibrated for these checkpoints. 
This experiment independently examines the behavioral relevance of the prospective groups.

\paragraph{Task results}
Table~\ref{tab:robot} reports all configurations.
Every high-deviation configuration significantly degrades
relative to its task's reference (two-sided Fisher tests:
$p<0.004$ for cups and $p<0.002$ for cubes).
No low-deviation comparison reaches significance at 0.05,
but these results do not establish performance preservation.
For HQQ W3, candidate-minus-reference success differences
are $-10.0$\,pp on cups (95\% Newcombe interval
$[-31.4,+11.0]$\,pp) and $-6.7$\,pp on cubes
($[-26.6,+13.8]$\,pp).
Both intervals leave non-inferiority within one percentage
point unresolved.
Figure~\ref{fig:hardware} illustrates individual recorded
episodes and its success counts come from the separate
task-evaluation collection.

\begin{table}[t]
\centering\footnotesize
\caption{Real-robot results. Cup: 20 trials per configuration,
bf16 19/20. Cubes: 30 trials per configuration, bf16 25/30.
Deviations are in millimeters (mm). }
\label{tab:robot}
\setlength{\tabcolsep}{3pt}
\begin{tabular}{lrrrr}
\toprule
& \multicolumn{2}{c}{Cup} & \multicolumn{2}{c}{Cubes}\\
\cmidrule(lr){2-3}\cmidrule(lr){4-5}
Configuration & $\dev$ & Success & $\dev$ & Success\\
\midrule
W8A8          & 0.14 & 19/20 & 0.12 & 26/30\\
W4A8-pg       & 0.97 & 18/20 & 0.74 & 24/30\\
W4A4          & 1.51 & 19/20 & 1.05 & 25/30\\
W4A8 (QServe) & 1.31 & 19/20 & 1.19 & 24/30\\
W3 (HQQ g64)  & 1.75 & 17/20 & 1.33 & 23/30\\
\midrule
W2 (RTN g64)  & 6.32  & 10/20 & 7.75  & 12/30\\
W2 (RTN g128) & 9.17  & 4/20  & 10.33 & 1/30\\
W2 (RTN pc)   & 66.10 & 0/20  & 91.60 & 0/30\\
\bottomrule
\end{tabular}
\end{table}

\begin{figure*}[t]
\centering
\includegraphics[width=\textwidth]{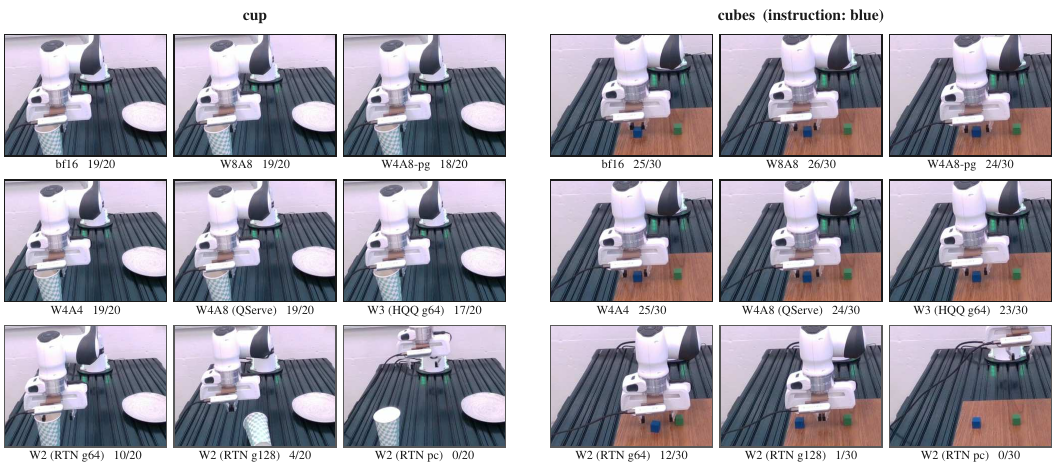}
\caption{One recorded episode per configuration on each task, showing contact or the corresponding attempted grasp:
cup (left) and instructed blue cube (right).
Low-deviation configurations (top two rows) close on the object. 
RTN W2 g128 pushes the cup over and closes beside the cube,
RTN W2 g64 pushes the cube aside during a mis-centered descent and closes beside it, and RTN W2 pc does not approach either object.}
\label{fig:hardware}
\end{figure*}

\paragraph{Quantization benefits on the robot}
A separate 45-episode instrumented collection supplies
resource measurements and video.
With other inference stages fixed, W4A4 reduces median
action-query latency from 2030 to 1485\,ms
($1.37\times$ speedup) and peak allocated memory from
5.19 to 2.93\,GB (approximately 44\%).
In the task-evaluation trials, it matches the reference's
observed cup success of 19/20
(Tables~\ref{tab:robot} and~\ref{tab:resource_robot}).
HQQ W3 reduces memory further but increases query latency.
The RTN W2 stress tests use fake quantization with bf16
storage and provide no weight-storage saving.

\begin{table}[t]
\centering\footnotesize
\caption{Resource measurements during Franka cup-task
execution on the RTX~5080 server. 
Query latency is the
median over recorded queries. 
Memory is peak allocated GPU
memory.}
\label{tab:resource_robot}
\setlength{\tabcolsep}{3pt}
\begin{tabular}{lrrr}
\toprule
Configuration & Query (ms) & Speedup & Memory (GB)\\
\midrule
bf16          & 2030 & $1.00\times$ & 5.19\\
W8A8          & 1725 & $1.18\times$ & 3.55\\
W4A8-pg       & 2019 & $1.01\times$ & 2.76\\
W4A4          & 1485 & $1.37\times$ & 2.93\\
W4A8 (QServe) & 1772 & $1.15\times$ & 2.73\\
W3 (HQQ g64)  & 2404 & $0.84\times$ & 2.70\\
\bottomrule
\end{tabular}
\end{table}

\section{Discussion and Limitations}
\label{sec:discussion}

Policy-specific calibration provides a behavioral basis for interpreting offline action deviation.
The held-out results show that this interpretation can identify both acceptable configurations and substantial degradation among the evaluated candidates.
Grouping and quantizer choice matter alongside precision, so nominal bit width alone does not fully interpret a configuration's behavioral effect.
PreDE uses measured policy changes to assess candidates produced by these different quantization choices.

Behavioral predictions must be considered alongside measured resource benefits.
On the robot, W4A4 reduces both latency and memory, whereas HQQ W3 reduces memory further but increases latency.
An accepted prediction therefore does not establish a deployment advantage by itself.
Configuration selection should consider the predicted task outcome together with the application's latency and memory constraints.
A deferred candidate may still merit closed-loop evaluation when its measured resource benefits are attractive.

Under our hypothesis, deferral corresponds to disagreement among thresholds consistent with the development labels.
Empirically, deferred candidates include both acceptable outcomes and a 33-point loss.
The interval therefore represents insufficient evidence for a prediction, rather than a region of uniformly small degradation.
Deferral also differs from an unresolved outcome:
the former withholds a prediction before evaluation, while the latter indicates that the observed result meets neither definite label.
Across the 28 held-out candidates, all 21 issued decisions match their labels, while the seven deferred candidates include three unresolved outcomes.
Decision coverage and outcome uncertainty thus describe different aspects of the evidence.

This distinction also guides additional evaluation.
Repeating offline measurements can reveal whether a candidate's deviation is stable compared with the thresholds, as examined in the seed-sensitivity analysis.
It cannot determine the task consequence of a stable deviation inside the deferred interval.
Closed-loop outcomes provide that behavioral evidence, although additional trials may still leave a result unresolved.
The appropriate next measurement therefore depends on whether uncertainty concerns the offline statistic, its behavioral interpretation, or the observed task outcome.

Several limitations remain.
Held-out prediction covers two checkpoints on LIBERO-10. 
The broader comparisons and physical grouping study provide complementary evidence.
Cosmos development and held-out evaluations differ in task coverage, so their pooled ordering is not a test under an identical task distribution.
The theoretical conclusions depend on the hypothesis, whose consistency with observed labels does not guarantee ordering on unseen configurations. 
Mean deviation also compresses errors across observations, action components, and chunk steps into a single number.
Two candidates with similar means can place their errors at different moments, including contact or grasp transitions where task consequences may differ.
Finite logs can miss such critical states and states reached only after quantization changes the controller, reflecting the distribution-shift challenge in sequential prediction~\cite{dagger}.
Changes to the policy or evaluation protocol may require recalibration.
Calibration and evaluation of deferred candidates require  closed-loop interaction, and finite trial counts leave some outcomes unresolved.
Future work should examine broader policy coverage, smaller calibration budgets, and measurements that retain information about errors at decision-critical states.

\section{Conclusion}
\label{sec:conclusion}

We presented PreDE, a policy-calibrated framework for predicting quantization-induced task degradation in world action models.
Using closed-loop results from development configurations, PreDE calibrates offline action deviations to accept, reject, or defer new candidate quantization configurations.
Under a within-setting label-ordering hypothesis, these decisions correspond to agreement among thresholds consistent with the development labels. 
Across 28 held-out configurations from two policies, PreDE issues 21 decisions (75\% coverage), all matching the observed acceptable or degraded labels.
A separate 450-trial real robot study connects prospectively identified high-deviation groups with significant task degradation.
On the robot, W4A4 quantization achieves a $1.37\times$ action-query speedup and approximately 44\% lower peak memory.
These results support policy-specific behavioral calibration for quantization configuration selection while identifying candidates that require further closed-loop evaluation.

\bibliographystyle{IEEEtran}
\bibliography{references}
\end{document}